%% file: bandits.tex
\pdfoutput=1

\documentclass{article} 
\usepackage{nips15submit_e, times}

\usepackage{amsmath, amsthm, amssymb,  pdfsync, paralist}

\usepackage[round]{natbib}
\let\chapter\section

\usepackage[ruled,vlined]{algorithm2e}

\usepackage{enumitem}

\newcommand\ee{\mathbf{e}}

\newcommand{\XX}{G}
\newcommand{\xx}{g}

\newcommand\lh{\hat{\xx}}
\newcommand\Lh{\hat{\XX}}

\newcommand\Lt{\tilde{\XX}}

\DeclareMathOperator*{\argmax}{arg\,max}

\newcommand{\Regret}{\mathrm{Regret}}

\newcommand{\PP}{\mathbb{P}}
\newcommand{\ind}{\mathbf{1}}

\newcommand{\EE}{\mathbb{E}}
\newcommand{\RR}{\mathbb{R}}

\newcommand{\f}{\Phi}
\newcommand{\tf}{{\tilde{\f}}}
\newcommand{\gtf}{\nabla {\tilde{\f}}}
\newcommand{\htf}{\nabla^2 {\tilde{\f}}}

\def\K{\mathcal{K}}

\def\II{\mathbb{I}}

\newcommand{\D}{\Delta}

\newtheorem{theorem}{Theorem}[section]
\newtheorem{lemma}[theorem]{Lemma}

\newtheorem{corollary}[theorem]{Corollary}
\newtheorem{conjecture}[theorem]{Conjecture}
\newtheorem{definition}[theorem]{Definition}

\newenvironment{framework}[1][htb]
  {
   \begin{algorithm}[#1]%
  }{\end{algorithm}}

\title{Fighting Bandits with a New Kind of Smoothness}

\author{
Jacob Abernethy \\
University of Michigan\\
\texttt{jabernet@umich.edu} \\
\And
Chansoo Lee \\
University of Michigan\\
\texttt{chansool@umich.edu} \\
\And
Ambuj Tewari \\
University of Michigan\\
\texttt{tewaria@umich.edu} \\
}

\nipsfinalcopy 
 \pdfoutput=1

\begin{document}

\maketitle

\input{abstract}
\input{intro}
\input{algorithm}
\input{ftrl_analysis}
\input{hazardrate}
\paragraph*{Acknowledgments.}
J. Abernethy acknowledges the support of NSF under CAREER grant IIS-1453304.
A. Tewari acknowledges the support of NSF under CAREER grant IIS-1452099.

\bibliographystyle{plainnat}
\bibliography{bandits}

\newpage
\input{appendix}

\end{document}

%% file: abstract.tex

\begin{abstract}

We define a novel family of algorithms for the adversarial multi-armed bandit problem, and provide a simple analysis technique based on convex smoothing. We prove two main results. First, we show that regularization via the \emph{Tsallis entropy}, which includes EXP3 as a special case, achieves the $\Theta(\sqrt{TN})$ minimax regret. Second, we show that a wide class of perturbation methods achieve a near-optimal regret as low as $O(\sqrt{TN \log N})$ if the perturbation distribution has a bounded hazard rate. 
For example, the Gumbel, Weibull, Frechet, Pareto, and Gamma distributions all satisfy this key property.

\end{abstract}

%% file: intro.tex

\section{Introduction}

The classic \emph{multi-armed bandit} (MAB) problem, generally attributed to the early work of \cite{Robbins52}, poses a generic online decision scenario in which an agent must make a sequence of choices from a fixed set of options. After each decision is made, the agent receives some feedback in the form of a loss (or gain) associated with her choice, but no information is provided on the outcomes of alternative options. The agent's goal is to minimize the total loss over time, and the agent is thus faced with the balancing act of both experimenting with the menu of choices while also utilizing the data gathered in the process to improve her decisions. The MAB framework is not only mathematically elegant, but useful for a wide range of applications including medical experiments design \citep{gittins1996quantitative}, automated poker playing strategies \citep{van2009monte}, and hyperparameter tuning \citep{pacula2012hyperparameter}.

Early MAB results relied on stochastic assumptions (e.g., IID) on the loss sequence \citep{gittins2011multi,LaiRobbins85,auer2002finite}. As researchers began to establish non-stochastic, \emph{worst-case} guarantees for sequential decision problems such as \emph{prediction with expert advice} \citep{LitWar94}, a natural question arose as to whether similar guarantees were possible for the bandit setting. The pioneering work of \citet*{AueCesFreSch03nonstochastic} answered this in the affirmative by showing that their algorithm EXP3 possesses nearly-optimal regret bounds with matching lower bounds. Attention later turned to the bandit version of \emph{online linear optimization}, and several associated guarantees were published the following decade \citep{McMahBlum04,FlaKalMcMah05,DanHay06robbing,DanHayKak07price,abernethy2012interior}.

Nearly all proposed methods have relied on a particular algorithmic blueprint; they reduce the bandit problem to the full-information setting, while using randomization to make decisions and to \emph{estimate} the losses. A well-studied family of algorithms for the full-information setting is \emph{Follow the Regularized Leader} (FTRL), which optimizes the objective function of the following form:
\begin{equation} \label{eq:ftrl}
	\mathop{\arg\min}_{x \in \mathcal{K}} \langle L, x  \rangle + \lambda R(x) 
\end{equation}
where $\K$ is the decision set, $L$ is (an estimate of) the cumulative loss vector, and $R$ is a \emph{regularizer}, a convex function with suitable curvature to stabilize the objective.
The choice of regularizer $R$ is critical to the algorithm's performance.
For example, the EXP3 algorithm \citep{auer2003using} regularizes with the \emph{entropy function} and achieves a nearly optimal regret bound when $\K$ is the probability simplex.
For a general convex set, however, other regularizers such as \emph{self-concordant barrier functions} \citep{abernethy2012interior} have tighter regret bounds.

Another class of algorithms for the full information setting is \emph{Follow the Perturbed Leader} (FTPL) \citep{KV-FTL} whose foundations date back to the earliest work in adversarial online learning \citep{Hannan57}. Here we choose a distribution $\mathcal{D}$ on $\RR^N$, sample a random vector $Z \sim \mathcal{D}$, and solve the following \emph{linear} optimization problem
\begin{equation} \label{eq:ftpl}
	\mathop{\arg\min}_{x \in \mathcal{K}} \langle L+Z, x \rangle
\end{equation}
FTPL is computationally simpler than FTRL due to the linearity of the objective, but it is analytically much more complex due to the randomness. For every different choice of $\mathcal{D}$, an entirely new set of techniques had to be developed \citep{devroye2013prediction,van2014follow}. \cite{rakhlin2012relax} and \citet{abernethy2014online} made some progress towards unifying the analysis framework. Their techniques, however, are limited to the full-information setting.



In this paper, we propose a new analysis framework for the multi-armed bandit problem that unifies the regularization and perturbation algorithms. The key element is a new kind of smoothness property, which we call \emph{differential consistency}. It allows us to generate a wide class of both optimal and near-optimal algorithms for the adversarial multi-armed bandit problem. We summarize our main results:
\begin{enumerate}

 	\item We show that regularization via the \emph{Tsallis entropy} leads to the state-of-the-art adversarial MAB algorithm, matching the minimax regret rate of \cite{audibert2009minimax} with a tighter constant. Interestingly, our algorithm fully generalizes EXP3.
 	\item We show that a wide array of well-studied noise distributions lead to near-optimal regret bounds (matching those of EXP3). Furthermore, our analysis reveals a strikingly simple and appealing sufficient condition for achieving $O(\sqrt{T})$ regret: the \emph{hazard rate} of the noise distribution must be bounded by a constant in the tail region. We conjecture that this requirement is in fact both necessary and sufficient.
 \end{enumerate} 





%% file: algorithm.tex

\section{Gradient-Based Prediction Algorithms for the Multi-Armed Bandit}

Let us now introduce the adversarial multi-armed bandit problem. On each round $t = 1,\ldots, T$, a learner must choose a distribution $p_t \in \D_N$ over the set of $N$ available actions. The adversary (Nature) chooses a loss vector $\xx_t \in [-1,0]^N$. The learner plays action $i_t$ sampled according to $p_t$ and suffers the loss $\xx_{t,{i_t}}$. The learner observes only a single coordinate $\xx_{t,{i_t}}$ and receives no information as to the values $\xx_{t,j}$ for $j \ne i_t$. 
This limited information feedback is what makes the bandit problem much more challenging than the full-information setting in which the entire $\xx_t$ is observed. 

The learner's goal is to minimize the \emph{regret}. Regret is defined to be the difference in the realized loss and the loss of the best fixed action in hindsight:
\begin{equation} \label{eq:regret}
	\Regret_T := \max_{i \in [N]} \sum_{t=1}^T (\xx_{t,i} - \xx_{t,i_t}).
\end{equation}
To be precise, we consider the \emph{expected} regret, where the expectation is taken with respect to the learner's randomization. 

\paragraph{Loss vs. Gain Note:} The maximization in \eqref{eq:regret} would imply that $\xx$ is strictly speaking a negative \emph{gain}. Nevertheless, we use the term \emph{loss}, as we impose the assumption that $\xx_t \in [-1,0]^N$ throughout the paper. 

\pagebreak

\subsection{The Gradient-Based Algorithmic Template}

We study a particular algorithmic template described in Framework~\ref{alg:gbpa_bandit}, which is  a slight variation of the Gradient Based Prediction Algorithm (GBPA) of \cite{abernethy2014online}. Note that the algorithm \begin{inparaenum}[(i)] \item maintains an \emph{unbiased estimate} of the cumulative losses $\Lh_t$, \item updates $\Lh_t$ by adding a single-round estimate $\lh_t$, and \item uses the gradient of a convex function $\tf$ as sampling distribution $p_t$.
\end{inparaenum}
 The choice of $\tf$ is flexible, it must be a differentiable convex function such that its gradient is always a probability distribution. 

Framework~\ref{alg:gbpa_bandit} may appear restrictive but it has served as the basis for much of the published work on adversarial MAB algorithms \citep{AueCesFreSch03nonstochastic, kujala2005following, neu2013efficient} mainly for two reasons.
First, the GBPA framework encompasses all FTRL and FTPL algorithms, which are the core techniques for sequential prediction algorithms \citep{abernethy2014online}. 
Second, 
although there is some flexibility, any unbiased estimation scheme would require some kind of inverse-probability scaling. Information theory tells us that the unbiased estimates of a quantity that is observed with only probabilty $p$ must necessarily involve fluctuations that scale as $O(1/p)$.

\begin{framework}
\caption{Gradient-Based Prediction Alg. (GBPA) Template for Multi-Armed Bandits.}
GBPA$(\tf)$: $\tf$ is a differentiable convex function such that $\gtf \in \Delta^{N}$ and $\nabla_i \tf > 0$ for all $i$. \\
\label{alg:gbpa_bandit}
Initialize $\Lh_0 =0$\\
\For{t = 1 to T}{
\textbf{Nature:} A loss vector $\xx_{t} \in [-1,0]^{N}$ is chosen by the Adversary\\
\textbf{Sampling:} Learner chooses $i_t$ according to the distribution $p(\Lh_{t-1}) = \nabla \Phi_{t}(\Lh_{t-1})$\\
\textbf{Cost:} Learner ``gains'' loss $\xx_{t,i_t}$\\
\textbf{Estimation:} Learner ``guesses'' $\lh_{t} := \frac{\xx_{t,i_t}}{p_{i_t}(\Lh_{t-1})}\ee_{i_t}$\\
\textbf{Update:} $\Lh_{t} = \Lh_{t-1} + \lh_t$ 
}
\end{framework}

\begin{lemma}
\label{lem:genericregret1} Define $\f(\XX) \equiv \max_{i} \XX_i$ so that we can write the expected
regret of GBPA$(\tf)$ as \[\EE\Regret_T = \f(\XX_T) - \textstyle \sum_{t=1}^{T}\langle \gtf(\Lh_{t-1}), \xx_t \rangle.\] Then, the expected regret of GBPA$(\tf)$ 
can be written as:
\begin{equation}
\label{eq:genericregret}
\EE\Regret_T \leq
 \underbrace{\tf(0) - \f(0)}_{\text{overestimation penalty}} 
+ \EE_{i_1, \ldots, i_{t-1}}\bigg[\underbrace{\f(\Lh_T) - \tf(\Lh_T)}_{\text{underestimation penalty}} 
+ \sum_{t = 1}^{T} \underbrace{\EE_{i_t}[D_{\tf}(\Lh_t, \Lh_{t-1})| \Lh_{t-1}]}_{\text{divergence penalty}}\bigg],
\end{equation}
where the expectations are over the sampling of $i_t$.
\end{lemma}

\begin{proof} Let $\tf$ be a valid convex function for GBPA. Consider GBPA$(\tf)$ run on the loss sequence $\xx_1, \ldots, \xx_T$. The algorithm produces a sequence of estimated losses $\lh_1, \ldots, \lh_T$. 
Now consider GBPA-FI$(\tf)$, which is GBPA$(\tf)$ run with the full information on the deterministic loss sequence $\lh_1, \ldots, \lh_T$ (there is no estimation step, and the learner updates $\Lh_{t}$ directly).
The regret of this run can be written as 
\begin{equation}
\label{eq:gbpa_fi}
	\Phi(\Lh_T) - \textstyle \sum_{t=1}^{T}\langle \gtf(\Lh_{t-1}), \lh_t \rangle
\end{equation}
and $\Phi(\XX_T) \leq \Phi(\Lh_T)$ by the convexity of $\Phi$. Hence, Equation \ref{eq:gbpa_fi} is an upper bound the regret. The rest of the proof is a fairly well-known result in online learning literature; see, for example, \citep[Theorem~11.6]{CesaBianchiLugosi06book} or \citep[Section~2]{abernethy2014online}. For completeness, we included the full proof in Appendix \ref{app:gbpa_regret}.  
\qedhere
\end{proof}

\subsection{A New Kind of Smoothness}

What has emerged as a guiding principle throughout machine learning is that the \emph{stability} of an algorithm leads to performance guarantees---that is, small modifications of the input data should not dramatically alter the output.
In the context of GBPA, algorithm's output (prediction in each time step) is by definition the dervative $\gtf$, and its stability corresponds to the Lipschitz-continuity of the gradient. \cite{abernethy2014online} proved that  a uniform bound the norm of $\htf$ directly gives a regret guarantee   for the full-information setting. 

In the bandit setting, however, a uniform bound on $\htf$ is insufficient; the regret (Lemma \ref{lem:genericregret1}) involves terms of the form $D_{\tf}(\Lh_{t-1} + \lh_{t}, \Lh_{t-1})$, where the incremental quantity $\lh_t$ can scale as large as \emph{the inverse of the smallest probability} of $p(\Lh_{t-1})$. 
What is needed is a stronger notion of the smoothness that bounds $\htf$ in correspondence with $\gtf$, and we propose the following definition:

\begin{definition}[Differential Consistency]
\label{def:diffcons}
	For constants $\gamma, C > 0$, we say that a convex function $\tf(\cdot)$ is \emph{$(\gamma,C)$-differentially-consistent} if for all $\XX \in (-\infty, 0]^N$,
\[
	\nabla^{2}_{ii}\tf(\XX) \leq C (\nabla_{i}\tf(\XX))^{\gamma}.
\]
\end{definition}

In other words, the rate in which we decrease $p_i$ should approach 0 as $p_i$ approaches 0. This gurantees that the algorithm continues to explore.
We now prove a generic bound that we will use in the following two sections to derive regret guarantees.
\begin{theorem}
\label{thm:divergence_bound}
Suppose $\tf$ is $(\gamma,C)$-differentially-consistent for constants $C,\gamma > 0$. Then divergence penalty at time $t$ in Lemma \ref{lem:genericregret1} can be upper bounded as:
	\[
	\EE_{i_t}[D_{\tf}(\Lh_t, \Lh_{t-1}) | \Lh_{t-1}] \leq
 \frac{C}{2} \sum_{i=1}^N \left(\nabla_{i}\tf(\Lh_{t-1})\right)^{\gamma-1} .
	\]
\end{theorem}
\vspace{-1em}
\begin{proof} 
For the sake of clarity, we drop the subscripts; we use $\Lh$ to denote the cumulative estimate $\Lh_{t-1}$, $\lh$ to denote the marginal estimate $\lh_t = \Lh_{t} - \Lh_{t - 1}$, and $g$ to denote the true loss $\xx_t$.

Note that by definition of Algorithm \ref{alg:gbpa_bandit}, $\lh$ is a sparse vector with one non-zero and non-positive coordinate $\lh_{i_{t}} = \xx_{t,i} / \nabla_{i}\tf(\Lh)$. Plus, $i_t$ is conditionally independent given $\Lh$.
For a fixed $i_t$, Let 
	\[h(r) := D_\tf(\Lh + r \lh / \|\lh\|, \Lh) = D_\tf(\Lh + r \ee_{i_t}, \Lh),\]
	so that $h''(r) = (\lh / \|\lh\|)^\top \nabla^2 \tf\left(\Lh + t \lh / \|\lh\|\right) (\lh / \|\lh\|) = \ee_{i_t}^{\top}\nabla^2 \tf\left( \Lh - t \ee_{i_t} \right)\ee_{i_t}.$
	Now we write:
\begin{align*}
\EE_{i_t}[D_\tf(\Lh + \lh, \Lh) | \Lh] &\textstyle = 
\sum_{i=1}^{N} \PP[i_t = i] \int_{0}^{\|\lh\|} \int_{0}^{s} h''(r) \ dr \ ds \\ 
& \textstyle=
   \sum_{i=1}^{N} \nabla_{i} \tf(\Lh) \int_{0}^{\|\lh\|}\int_{0}^{s} \ee_i^\top \nabla^2 \tf\left( \Lh - r \ee_i \right)\ee_i \ dr \ ds \\
& \textstyle \leq \sum_{i=1}^{N} \nabla_{i} \tf(\Lh) \int_{0}^{\|\lh\|}\int_{0}^{s}  C \left(\nabla_{i} \tf (\Lh - r \ee_i )\right)^{\gamma} \ dr \ ds \\
& \textstyle\leq  \sum_{i=1}^{N} \nabla_{i} \tf(\Lh) \int_{0}^{\|\lh\|}\int_{0}^{s} C \left(\nabla_{i} \tf (\Lh) \right)^{\gamma} \ dr \ ds \\
& \textstyle= C \sum_{i=1}^{N} \left(\nabla_{i} \tf(\Lh)\right)^{1+\gamma} \int_{0}^{\|\lh\|}\int_{0}^{s} \ dr \ ds  \\
& \textstyle= \frac{C}{2}\sum_{i=1}^N  \left(\nabla_{i}\tf(\Lh)\right)^{\gamma-1} \xx_{i}^2.
\end{align*}
The first inequality is by the supposition. The second inequality is due to the convexity of $\tf$ which guarantees that $\nabla_i$ is an increasing function in the $i$-th coordinate; this step critically depends on the \emph{loss-only} assumption that $\xx$ is always non-positive.
\end{proof}

%% file: ftrl_analysis.tex
\section{A Minimax Bandit Algorithm via Tsallis Smoothing}


\cite{AueCesFreSch03nonstochastic} proved that their EXP3 algorithm achieves $O(\sqrt{T N \log N})$ regret and that any multi-armed bandit algorithm suffers $\Omega(\sqrt{T N})$ regret. A few years later, \cite{audibert2009minimax} resolved this gap with 
 Implicitly Normalized Forecaster (INF), which later was shown to be equivalent to Mirror Descent \citep{audibert2011minimax} on the probability simplex. EXP3 corresponds to INF with potential function $\psi(x) = \exp(\eta x)$, while using $\psi(x) = (-\eta x)^{-q}$ with $q > 1$ gives an optimal algorithm that has regret at most $2\sqrt{2TN}$ \citep[Theorem~5.7]{bubeck2012regret}.

What we present in this section is essentially a reformulation of a particular subfamily of INF, which includes INF with the above two potential functions. Our reformulation leads to a very simple and intuitive analysis based on \emph{differential consistency}, and a natural interpolation between the two seemingly unrelated potential functions.

Let us first note that EXP3 is an instance of GBPA where the potential function $\tf(\cdot)$ is the Fenchel conjugate of the \emph{Shannon entropy}. For any $p \in \D_N$, the (negative) Shannon entropy is defined as $H(p) := \sum_i p_i \log p_i$, and its Fenchel conjugate is $H^\star(\XX) = \sup_{p \in \D_N} \{\langle p, \XX \rangle - \eta H(p)\}.$ In fact, we have a closed-form expression for the supremum: $H^\star(\XX) = \frac 1 \eta \log \left(\sum_i \exp(\eta \XX_i) \right).$ By inspecting the gradient of the above expression, it is easy to see that EXP3 chooses the distribution $p_t = \nabla H^\star(\XX)$ every round.


Now we will replace the Shannon entropy with the \emph{Tsallis entropy}\footnote{More precisely, the function we give here is the \emph{negative} Tsallis entropy according to its original definition.} \citep{tsallis1988possible}, defined as:
\[S_\alpha(p) = \frac{1}{1 - \alpha} \left(1 - \sum_{i=1}^{N} p_{i}^{\alpha} \right) \ \ \mbox{for $0 < \alpha < 1$}.\]
Interestingly, the Shannon entropy is an asymptotic special case of the Tsallis entropy, i.e.,
\[
	S_\alpha(\cdot) \to H(\cdot) \quad \quad \text{ as } \alpha \to 1.
\]
\begin{theorem}\label{thm:tsallis}
	Let $\tf(\XX) = \max_{p \in \D_N}\{\langle p, \XX \rangle - \eta S_\alpha(p) \}$. Then the GBPA$(\tf)$ has regret at most
	\begin{equation} \label{eq:tsallisbound}
		\EE\Regret \leq
		\underbrace{\eta \frac{N^{1 - \alpha} - 1}{1-\alpha}}_{\text{overestimation penalty}} 
		+ \underbrace{\frac{N^\alpha T}{2 \eta \alpha}}_{\text{divergence penalty}}.
	\end{equation}
\end{theorem}
Before proving the theorem, we note that it immediately recovers the EXP3 upper bound as a special case $\alpha \to 1$. An easy application of L'H\^opital's rule shows that as $\alpha \to 1$, $\frac{N^{1 - \alpha} - 1}{1-\alpha} \to \log N$ and $N^\alpha/\alpha \to N$. Choosing $\eta = \sqrt{(N \log N)/T}$, we see that the right-hand side of \eqref{eq:tsallisbound} tends to $2\sqrt{T N \log N}$. However the choice $\alpha \to 1$ is clearly not the optimal choice, as we show in the following statement, which directly follows from the theorem once we see that $N^{1-\alpha} - 1 < N^{1-\alpha}$.
\begin{corollary}
	For any $\alpha \in (0,1)$, if we choose $\eta = \sqrt{\frac{T(1-\alpha)}{2\alpha}} N^{\alpha-\frac{1}{2}}$ then we have
	\[
		\textstyle \EE\Regret \leq \sqrt{\frac{2 T N}{\alpha(1-\alpha) }}.
	\]
	In particular, the choice of $\alpha = \frac 1 2$ gives a regret of no more than $2\sqrt{2 T N}$, recovering \citep[Theorem~5.7]{bubeck2012regret}.
\end{corollary}
\begin{proof}[Proof of Theorem~\ref{thm:tsallis}]
	We will bound each penalty term in Lemma \ref{lem:genericregret1}. Since $S_{\alpha}$ is non-positive, the \textbf{underestimation penalty} is upper bounded by 0 and the \textbf{overestimation penalty} is at most $(-\min S_{\alpha})$. The minimum of $S_{\alpha}$ occurs at $(1/N, \ldots, 1/N)$. Hence,
	\begin{equation}
	\label{eq:ftrl_underestimation}
		\text{(overestimation penalty)} 
		 \leq -\frac{\eta}{1-\alpha}\left(1 - \sum_{i=1}^{N}\frac{1}{N^\alpha}\right) \leq \eta \frac{N^{1 - \alpha} - 1}{1-\alpha} .
	\end{equation}

	Now it remains to upper bound the \textbf{divergence penalty}. Straightforward calculus gives 
	\[\nabla^2 S_{\alpha} (p) = \eta \alpha \text{diag}(p_1^{\alpha - 2}, \ldots,p_N^{\alpha - 2}).\]
 Let $\II_{\D_N}(\cdot)$ be the function where
 \text{$\II_{\D_N}(x) = 0$ for $x \in \D_N$ and $\II_{\D_N}(x) = \infty$ for $x \notin \D_N.$}
	Define a function $\hat{S}_{\alpha}(\cdot) := S_{\alpha}(\cdot) + \II_{\D_N}(\cdot)$, which is the convex conjugate of $\tf$.  Following the setup of \cite{penot1994sub}, $\nabla^2 S_{\alpha} (p)$ is a \emph{sub-hessian} of $\hat{S}_{\alpha}(p)$. 
	We now apply Proposition 3.2 of the same reference. Let $(p_\XX, \XX)$ be a pair such that $\gtf(\XX) = p_\XX$. 
	Since $\nabla^2 S_{\alpha} (p)$ is invertible, it follows that
	$(\nabla^{2}S_{\alpha}(p_\XX))^{-1}$ is a \emph{super-hessian} of $\tf$ at $\XX$.
	Hence, for any $\XX$,
		\[\nabla^2 \tf(\XX) \preceq 	 (\eta \alpha)^{-1}\mathrm{diag}\big((p_{\XX})_1^{2-\alpha}, \ldots, (p_{\XX})_N^{2-\alpha}(\XX)\big).\]
	That is, $\tf$ is $(2-\alpha,(\eta\alpha)^{-1})$-differentially-consistent, and thus applying Theorem \ref{thm:divergence_bound} gives
	\[ D_{\tf}(\Lh_{t}, \Lh_{t-1}) \leq
	 (2 \eta \alpha)^{-1}
	 \sum_{i=1}^{N} \Big(p_i(\Lh_{t-1})\Big)^{1-\alpha} .
	 \]
	Since the $\frac 1 \alpha$-norm and the $\frac 1 {1-\alpha}$-norm are dual to each other, we can apply H\"older's inequality to any probability distribution $p_1, \ldots, p_N$ and obtain
	\[
		\sum_{i=1}^N p_i^{1 - \alpha} 
		= \sum_{i=1}^N p_i^{1- \alpha} \cdot 1 \leq \left(\sum_{i=1}^N p_i^{\frac{1-\alpha}{1-\alpha}} \right)^{1-\alpha}\left(\sum_{i=1}^N 1^{\frac 1 \alpha} \right)^{\alpha} 
		= (1)^{1-\alpha} N^{\alpha} = N^\alpha.
	\]
	So, the divergence penalty is at most $(2\eta\alpha)^{-1} N^\alpha$, which completes the proof.
\end{proof}

%% file: hazardrate.tex

\section{Near-Optimal Bandit Algorithms via Stochastic Smoothing}
\label{sec:ftpl}
Let $\mathcal{D}$ be a continuous distribution over an unbounded support with probability density function $f$ and cumulative density function $F$. Consider the GBPA with potential function of the form: 
\begin{equation}
	\label{eq:potential_ftpl}
	\tf(\XX; \mathcal{D}) = \EE_{Z_{1}, \ldots, Z_{N} \stackrel{\text{iid}}{\sim} \mathcal{D}} \max_{i} \{\XX_{i} + Z_{i}\}
\end{equation}
which is a \emph{stochastic smoothing} of $(\max_{i} G_{i})$ function.
Since the max function is convex, $\tf$ is also convex. By \cite{bertsekas1973}, we can swap the order of differentiation and expectation:
\begin{equation}
	\label{eq:grad_ftpl}
	\gtf(\XX; \mathcal{D}) = \EE_{Z_{1}, \ldots, Z_{N} \stackrel{\text{iid}}{\sim} \mathcal{D}} e_{i^*}, \text{ where } i^* = \argmax_{i = 1, \ldots, N} \{\XX_{i} + Z_{i}\}.
\end{equation}
Even if the function is not differentiable everywhere, the swapping is still possible with any subgradient under some mild conditions. Hence, the ties between coordinates (which happen with probability zero anyways) can be resolved in an arbitrary manner.
 It is clear that $\gtf$ is in the probability simplex, and note that
\begin{align}
	\frac{\partial \tf}{\partial \XX_i} &= \EE_{Z_1,\ldots,Z_{N}} \ind\{\XX_{i} + Z_{i} > \XX_{j} + Z_{j}, \forall j \neq i\} \nonumber \\
				&= \EE_{\Lt_{j^*}}[\PP_{Z_i}[Z_{i} > \Lt_{j^*} - \XX_{i}]] = \EE_{\Lt_{j^*}}[1 - F(\Lt_{j^*} - \XX_{i})] \label{eq:dphi_ftpl}
\end{align}
where $\Lt_{j^*} = \max_{j \ne i} \XX_j + Z_j$. The unbounded support condition guarantees that this partial derivative is non-zero for all $i$ given any $\XX$. So, $\tf(\XX; \mathcal{D})$ satisfies the requirements of Algorithm \ref{alg:gbpa_bandit}.

Despite the fact that perturbation-based algorithms provide a natural randomized decision strategy, they have seen little applications mostly because they are hard to analyze. But one should expect general results to be within reach: the EXP3 algorithm, for example, can be viewed through the lens of perturbations, where the noise is distributed according to the Gumbel distribution. Indeed, an early result of \cite{kujala2005following} showed that a near-optimal MAB strategy comes about through the use of exponentially-distributed noise, and the same perturbation strategy has more recently been utilized in the work of \cite{neu2013efficient} and \cite{NIPS2014_5462}. However, a more general understanding of perturbation methods has remained elusive. For example, would Gaussian noise be sufficient for a guarantee? What about, say, the Weibull distribution?

\subsection{Connection to Follow the Perturbed Leader}
\label{sec:ftpl:ftpl}
The sampling step of the bandit GBPA (Framework~\ref{alg:gbpa_bandit}) with a stochastically smoothed function (Equation~\ref{eq:potential_ftpl}) can be done efficiently. Instead of evaluating the expectation (Equation~\ref{eq:grad_ftpl}), we simply take a random sample. In fact, this is equivalent to \textbf{Follow the Perturbed Leader Algorithm (FTPL)} \citep{KV-FTL} applied to the bandit setting. On the other hand, the estimation step is hard because generally there is no closed-form expression for $\gtf$.

To address this issue, \cite{neu2013efficient} proposed Geometric Resampling (GR). 
GR uses an iterative resampling process to estimate $\gtf$.
They showed that if we stop after $M$ iterations, the extra regret due to the estimation bias is at most $\frac{NT}{eM}$ (additive term).
That is, all our GBPA regret bounds in this section hold for the corresponding FTPL algorithm with an extra additive $\frac{NT}{eM}$ term..
This term, however, does not affect the asymptotic regret rate as long as $M = NT$, because the lower bound for any algorithm is of the order $\sqrt{NT}$.



\subsection{Hazard Rate analysis}
In this section, we show that the performance of the GBPA($\tf(G;\mathcal{D})$) can be characterized by the \emph{hazard function} of the smoothing distribution $\mathcal{D}$. The hazard rate is a standard tool in survival analysis to describe failures due to aging; for example, an increasing hazard rate models units that deteriorate with age while a decreasing hazard rate models units that improve with age (a counter intuitive but not illogical possibility). 
To the best of our knowledge, the connection between hazard rates and design of adversarial bandit algorithms has not been made before.

\begin{definition}[Hazard rate function]
	Hazard rate function of a distribution $\mathcal{D}$ is \[h_\mathcal{D}(x) := \frac{f(x)}{1-F(x)}\]
\end{definition}
\vspace{-1em}
For the rest of the section, we assume that $\mathcal{D}$ is unbounded in the direction of $+\infty$, so that the hazard function is well-defined everywhere. This assumption is for the clarity of presentation and can be easily removed (Appendix \ref{app:Drelaxation}).

\begin{theorem}
	The regret of the GBPA with $\tf(\XX) = \EE_{Z_1, \ldots, Z_n \sim D} \max_{i} \{\XX_{i} +  \eta Z_{i}\}$ is at most:
	\[\underbrace{\eta \EE_{Z_1, \ldots, Z_n \sim D} \left[\max_{i} Z_{i}\right]}_{\text{overestimation penalty}} + \underbrace{\frac{N(\sup h_\mathcal{D})}{\eta}T}_{\text{divergence penalty}} \]
\end{theorem}
\vspace{-1em}
\begin{proof}
We analyze each penalty term in Lemma \ref{lem:genericregret1}. Due to the convexity of $\Phi$, the underestimation penalty is non-positive. The overestimation penalty is clearly at most $\EE_{Z_1, \ldots, Z_n \sim D} [\max_{i} Z_{i}]$, and Lemma  \ref{lem:divergence} proves the $N(\sup h_\mathcal{D})$ upper bound on the divergence penalty.

It remains to prove the tuning parameter $\eta$. Suppose we scale the perturbation $Z$ by $\eta > 0$, i.e., we add $\eta Z_{i}$ to each coordinate. It is easy to see that $\EE[ \max_{i=1,\ldots,n} \eta X_i] = \eta \EE[ \max_{i=1,\ldots,n} X_i]$. For the divergence penalty, let $F_\eta$ be the CDF of the scaled random variable. Observe that $F_\eta(t) = F(t/\eta)$ and thus $f_\eta(t) = \frac{1}{\eta}f(t/\eta)$. Hence, the hazard rate scales by $1/\eta$, which completes the proof.
\end{proof}

\begin{lemma}
\label{lem:divergence}
The divergence penalty of the GBPA with $\tf(\XX) = \EE_{Z_1, \ldots, Z_n \sim D} \max_{i} \{\XX_{i} +  \eta Z_{i}\}$ is at most $N (\sup h_\mathcal{D})$ each round.
\end{lemma}
\begin{proof}
Recall the gradient expression in Equation \ref{eq:dphi_ftpl}.
We upper bound the $i$-th diagonal entry of the Hessian, as follows:
\begin{align}
	\nabla^2_{ii}\tf(\XX) &= \frac{\partial}{\partial \XX_i} \EE_{\Lt_{j^*}}[1 - F(\Lt_{j^*} - \XX_{i})] 
	=  \EE_{\Lt_{j^*}}\left[\frac{\partial}{\partial \XX_i}(1 - F(\Lt_{j^*} - \XX_{i}))\right] 
	=  \EE_{\Lt_{j^*}}f(\Lt_{j^*} - \XX_{i}) \notag \\
	&= \EE_{\Lt_{j^*}}[ h(\Lt_{j^*} - \XX_{i})(1 - F(\Lt_{j^*} - \XX_{i}))] \label{eq:divergence_hazard}\\
	&\leq (\sup h) \EE_{\Lt_{j^*}}[1 - F(\Lt_{j^*} - \XX_{i})] \notag \\
	&= (\sup h) \nabla_i(\XX) \notag
\end{align}
where $\Lt_{j^*} = \max_{j \ne i} \{\XX_j + Z_j\}$ which is a random variable independent of $Z_i$. We now apply Theorem \ref{thm:divergence_bound} with $\gamma = 1$ and $C = (\sup h)$ to complete the proof.
\end{proof}
\begin{corollary}
	Follow the Perturbed Leader Algorithm with distributions in Table \ref{tab:distributions} (restricted to a certain range of parameters),  combined with Geometric Resampling (Section \ref{sec:ftpl:ftpl}) with $M = \sqrt{NT}$,	 has an expected regret of order $O(\sqrt{TN \log N})$.
\end{corollary}
Table \ref{tab:distributions} provides the two terms we need to bound. We derive the third column of the table in Appendix \ref{app:distributions} using Extreme Value Theory \citep{embrechts1997modelling}. Note that our analysis in the proof of Lemma \ref{lem:divergence} is quite tight; the only place we have an inequality is where we upper bound the hazard rate. It is thus reasonable to pose the following conjecture:
 \begin{conjecture}
   If a distribution $\mathcal{D}$ has a monotonically increasing hazard rate $h_\mathcal{D}(x)$ that does not converge as $x \to +\infty$ (e.g., Gaussian), then there is a sequence of losses that will incur at least a linear regret.
 \end{conjecture}
 The intuition is that if adversary keeps incurring a high loss for the $i$-th arm, then with high probability $\Lt_{j^*} - \XX_{i}$ will be large. So, the expectation in Equation \ref{eq:divergence_hazard} will be dominated by the hazard function evaluated at large values of $\Lt_{j^*} - \XX_{i}$.

\input{distributions}

%% file: distributions.tex

\begin{center}
\begin{table}
\begin{tabular}{l|l|p{11em}|l}
	Distribution & $\sup_x h_\mathcal{D}(x)$ & $\EE[\max_{i=1}^N Z_{i}]$ & $O(\sqrt{TN\log N})$ Param.\\
	\hline
	\hline
	Gumbel($\mu=1, \beta=1$) & 1 as $x \to 0$ & $\log N +\gamma_0$ & N/A\\
	\hline
	Frechet ($\alpha > 1$) & at most $2\alpha$ & $N^{1/\alpha} \Gamma(1 - 1/\alpha)$ &$\alpha=\log N$\\
	\hline
	Weibull*$(\lambda =1, k \leq 1)$ & $k$ at $x = 0$ & $O(\left(\frac{1}{k}\right)!(\log N)^{\frac{1}{k}})$ & $k = 1$ (Exponential) \\
	\hline
	Pareto*$(x_m = 1, \alpha)$ & $\alpha$ at $x = 0$ & $\alpha N^{1/\alpha} / (\alpha - 1)$ & $\alpha = \log N$\\
	\hline
	Gamma$(\alpha \geq 1, \beta)$ & $\beta$ as $x \to \infty$ & $\log N + (\alpha -1)\log \log N - \log \Gamma(\alpha) + \beta^{-1}\gamma_0$ & $\beta = \alpha = 1$ (Exponential)
\end{tabular}
\caption{\emph{Distributions that give $O(\sqrt{T N \log N})$ regret FTPL algorithm.} The parameterization follows Wikipedia pages for easy lookup. We denote the Euler constant ($\approx 0.58$) by $\gamma_0$. Distributions marked with (*) need to be slightly modified using the conditioning trick explained in Appendix \ref{app:conditioning}. The maximum of Frechet hazard function has to be computed numerically \citep[p.~47]{elsayed2012reliability} but elementary calculations show that it is bounded by $2\alpha$ (Appendix \ref{app:frechet}).
\label{tab:distributions}}
\end{table}
\end{center}
\vspace{-2em}

%% file: appendix.tex

\appendix

\section{Proof of the GBPA Regret Bound (Lemma \ref{lem:genericregret1})}
\label{app:gbpa_regret}
{\IncMargin{1em}
\begin{algorithm}[t]
\caption{Gradient-Based Prediction Algorithm (GBPA) for Full Information Setting}
Input: $\tf$, a differentiable convex function such that $\gtf \in \Delta^{N}$ and $\nabla_i \tf > 0$ for all $i$. \\
\label{alg:gbpa}
Initialize $\XX_0 =0$\\
\For{t = 1 to T}{
\textbf{Sampling:} The learner chooses arm $i_t$ with probability $p_i(\Lh_{t-1}) = \nabla_i \Phi_{t}(\Lh_{t-1})$\\
Adversary chooses a loss vector $\xx_{t} \in [-1,0]^{N}$ and learner pays $\xx_{t,i}$ \\
Update $\XX_{t} = \XX_{t-1} + \xx_t$ 
}
\end{algorithm}
}

\begin{lemma}
	The expected regret of Algorithm \ref{alg:gbpa} can be written as:
	\[\EE\Regret =
 \underbrace{\tf(0) - \f(0)}_{\text{overestimation penalty}} 
+ \underbrace{\f(\XX_T) - \tf(\XX_T)}_{\text{underestimation penalty}} 
+ \sum_{t = 1}^{T} \underbrace{D_{\tf}(\XX_t, \XX_{t-1})}_{\text{divergence penalty}}
\]
\end{lemma}

\begin{proof}
Note that since $\f_0(0) = 0$,
\begin{align*}
\tf(\XX_T) &= \underbrace{\big(\tf(0) - \f_0(0)\big)}_{\text{overestimation penalty}} + \sum_{t = 1}^{T} \tf(\XX_{t}) - \tf(\XX_{t - 1}) \\
&= \underbrace{\big(\tf(0) - \f_0(0)\big)}_{\text{overestimation penalty}}
+ \sum_{t = 1}^{T} \langle  \nabla \tf(\XX_{t-1}), \ell_{t}) \rangle 
 + D_{\tf}(\XX_{t}, \XX_{t - 1})
\end{align*}

Therefore,
\begin{align*}
\EE\Regret &\overset{\text{def}}{=} \EE\left[\f(\XX_T) -  \sum_{t=1}^{T} \langle \tf(\XX_{t-1}), \xx_t \rangle\right]\\
&= \EE\left[\underbrace{\big(\f(\XX_T) - \tf(\XX_T)\big)}_{\text{underestimation penalty}} 
+ \tf(\XX_T) - \sum_{t=1}^{T} \langle \tf(\XX_{t-1}), \xx_t \rangle
  \right]  \\
 &=\EE\left[
\underbrace{\big(\f(\XX_T) - \tf(\XX_T)\big)}_{\text{underestimation penalty}}
+ \underbrace{\big(\tf(0) - \f_0(0)\big)}_{\text{overestimation penalty}}
 + D_{\tf}(\XX_{t}, \XX_{t - 1})
 \right]\\
\end{align*}
\end{proof}


\section{Relaxing Assumptions on the Distribution}
\label{app:Drelaxation}
\subsection{Mirroring trick for extending the support}

Let $X$ have support on $x>0$ with density $f$ and CDF $F$. Let us define $Y$ by mirroring the density of $X$ around zero, i.e., $Y$ has density $g(y) = \tfrac{1}{2} f(|y|)$
and CDF $G(y) = \tfrac{1}{2}\left(1+\mathrm{sign}(y)F(|y|)\right)$. Note that
$|Y|$ is distributed as $X$ and hence, 
\[
\EE[\max_i Y_i ] \le \EE[\max_i |Y_i|] = \EE[\max_i X_i] .
\]
The hazard $h_Y(y)$ for $y \ge 0$ is $f(y)/(1-F(y))$ and for $y < 0$ is $f(-y)/(1+F(-y)) \le F(-y)/(1-F(-y))$. Therefore,
\[
\sup_{y} h_Y(y) = \sup_{x>0} h_X(x) .
\]
This proves the following lemma.

\begin{lemma}\label{lem:mirrordist}
If a random variable $X$ has support on the non-negative reals with density $f(x)$ and we define $Y$ as the mirrored version
with density $g(y) = \tfrac{1}{2} f(|y|)$. Then, we have
\[
\EE[\max_i Y_i ] \le \EE[\max_i X_i],
\]
\[
\sup_{y} h_Y(y) = \sup_{x>0} h_X(x) 
\]
where $h_X, h_Y$ are hazard rates of $X, Y$ respectively.
\end{lemma}
 
\subsection{Conditioning trick for unbounded hazard rate near zero}
\label{app:conditioning}

Suppose $F(x)$ is the CDF of a random variable $X$ whose hazard rate is bounded for $x \ge 1$ but blows up near zero. Then define $Y$ as $X$ conditioned on $X \ge 1$.
That is, $Y$ has CDF, for $y > 0$:
\[
G(y) = P(X \ge 1+y|X > 1) = \frac{F(1+y)-F(1)}{1-F(1)}
\]
and density $g(y) = f(1+y)/(1-F(1)), y > 0$. So the hazard rate $h_Y(y)$ is $g(y)/(1-G(y)) = f(1+y)/(1-F(1+y)) = h_X(1+y)$. Therefore,
\[
\sup_{y > 0} h_Y(y) = \sup_{x > 1} h_X(x)
\]
which makes the hazard rate of $Y$ now bounded. This we have proved the lemma below.

\begin{lemma}\label{lem:hazard_fix}
If a hazard function of $X$ is bounded for $x > 1$ and blows up only for small values of $x$ then we can condition on $X > 1$ to define a new random variable
whose hazard rate is now bounded.
\end{lemma}
The same technique can be applied for any arbitrary constant other than $1$, but for the family of random variables we considered, it suffices to condition on $X \geq 1$.

\section{Detailed derivation of extreme value behavior}
\label{app:distributions}

\subsection{Maximum of iid Gumbel}
\label{app:distributions:gumbel}
The CDF of the Gumbel distribution is $\exp(-\exp(-x))$ and the expected value is $\gamma_0$, the Euler (Euler-Mascheroni) constant. Thus,
the CDF of the maximum of $n$ iid Gumbel random variables is $(\exp(-\exp(-x)))^N = \exp(-\exp(-(x-\log N)))$ which is also Gumbel but with the mean increased by $\log N$.

\subsection{Maximum of iid Frechet}
The CDF of Frechet is $\exp(-x^{-\alpha})$ and it has mean $\Gamma(1-\frac{1}{\alpha})$ as long as $\alpha > 1$ (otherwise it is infinite). Hence, the CDF of the maximum of $N$ iid Frechet random variables is 
\[(\exp(-x^{-\alpha}))^N = \exp(-N x^{-\alpha}) =  \exp\left(-\left(\frac{x}{N^\frac{1}{\alpha}}\right)^{-\alpha}\right)\]
 which is also Frechet but with mean scaled by $N^{1/\alpha}$.

\subsection{Maximum of iid Weibull}\label{sec:weibullcalc}

Let $X_i$ have modified Weibull distribution with CDF $1-\exp(-(x+1)^k+1)$. Thus, $P(\max_i X_i > t) \le NP(X_1 > t) = N\exp(-(t+1)^k+1)$.
For non-negative random variable $X$ and any $u > 0$, we have,
\[
\EE[X] = \int_{0}^{\infty} P(X > x) dx \le u + \int_{u}^\infty P(X > x)dx .
\]
Assume $k=1/m$ where $m \ge 1$ is a positive integer. Therefore, 
\begin{align*}
\EE[\max_i X_i] &\le u + \int_{u}^\infty N \exp(-(x+1)^k+1)dx \\
&\le u + 3 N \int_{u}^\infty \exp(-(x+1)^k)dx \\
&= u + 3 N \int_{u+1}^\infty \exp(-x^{1/m})dx \\
&=u + 3  N m \Gamma(m, (1+u)^{1/m} ) dx
\end{align*}
where $\Gamma(m, x)$ is the incomplete Gamma function that for a positive integer $m$ and $x>1$ simplifies to 
\begin{align*}
\Gamma(m, x) &= (m-1)! e^{-x} \sum_{k=0}^{m-1} \frac{x^k}{k!} \le (m-1)! e^{-x} \sum_{k=0}^{m-1} \frac{x^{m}}{k!} \\
&= (m-1)! e^{-x} x^{m} \sum_{k=0}^{m-1} \frac{1}{k!} \le (m-1)! e^{-x} x^{m} \sum_{k=0}^{\infty} \frac{1}{k!} \\
&\le 3(m-1)! e^{-x} x^{m} .
\end{align*}
Plugging this back above, we get, for any $u > 0$,
\[
\EE[\max_i X_i] \le u + 9 N m! e^{-(1+u)^{1/m}} (1+u) .
\]
Now choose $u = \log^m N + 1$ to get
\[
\EE[\max_i X_i] \le \log^m N + 9 N m! \frac{\log^m N}{N} \le 10 \, m! \log^m N.
\]
\subsection{Maximum of iid Gamma}
Let $Y$ be the maximum of $N$ iid Gamma$(\alpha,\beta)$ ramdom variables. Then, $\frac{Y - d_N}{c_N}$ follows Gumbel distribution, where $c_N = \beta^{-1}$ and $d_N = \beta^{-1}(\log N + (\alpha - 1)\log \log N - \log \Gamma(\alpha))$. In the language of extreme value theory, Gamma distribution belongs to the maximum domain of attraction of Gumbel distribution with parameters \citep{embrechts1997modelling}.
 As mentioned in Section \ref{app:distributions:gumbel}, Gumbel distribution has mean $\gamma_0$.

\subsection{Maximum of iid Pareto}\label{sec:paretocalc}
Let $X_i$ have modified Pareto distribution with CDF $1-1/(1+x)^\alpha$. Thus, $P(\max_i X_i > t) \le NP(X_1 > t) = N/(1+x)^\alpha$.
For non-negative random variable $X$ and any $u > 0$, we have,
\[
\EE[X] = \int_{0}^{\infty} P(X > x) dx \le u + \int_{u}^\infty P(X > x)dx .
\]
Therefore, for $\alpha > 1$,
\begin{align*}
\EE[\max_i X_i] &\le u + \int_{u}^\infty \frac{N}{(1+x)^\alpha} dx \\
&= u + \frac{N}{(\alpha-1)(1+u)^{\alpha-1}} .
\end{align*}
Setting $u = N^{1/\alpha} - 1$ gives the bound
\[
\EE[\max_i X_i]  \le \frac{\alpha}{\alpha -1} N^{1/\alpha} .
\]




\section{Hazard Functions of Modified Distributions and the Frechet Case}
\label{app:frechet}

\subsection{Pareto distribution}
Using the conditioning trick, we consider, for $\alpha > 1$ (otherwise mean is infinite), the modified Pareto distribution with pdf $f(x) = \frac{\alpha}{(x + 1)^{\alpha+1}}$
supported on $(0, \infty)$. Its CDF is
$1 - 1 / (x + 1)^\alpha$.
Its hazard function is $h(x) = \frac{\alpha}{x + 1}$
which decreases in $x$ and is bounded by $\alpha$. 
Expected maximum of $N$ iid Pareto random variables is bounded by $\alpha N^{1/\alpha} /(\alpha-1)$ (see Appendix~\ref{sec:paretocalc}).
This gives a regret bound of $\sqrt{NT} \sqrt{\alpha^2 N^{1/\alpha} /(\alpha-1)}$.

\subsection{Frechet distribution}
The CDF of Frechet is $\exp(-x^{-\alpha}), x > 0$ where $\alpha > 0$ is a shape parameter. The hazard function of Frechet distribution is $h(x) = \alpha x^{-\alpha-1} \frac{\exp(-x^{-\alpha})}{1-\exp(-x^{-\alpha})}$ which is 
hard to optimize analytically but can be upper bounded, for $\alpha > 1$, via elementary calculations given below, by $2\alpha$. 
The CDF of the maximum of $N$ iid Frechet random variables is $\exp(-(x/N^{1/\alpha})^{-\alpha})$ which is also
Frechet (but with mean scaled by $N^{1/\alpha}$) with expected value $N^{1/\alpha} \Gamma(1-\frac{1}{\alpha})$ (as long as $\alpha > 1$, otherwise expectation is infinite).
Thus, the regret bound we get is $O\left(\sqrt{N T} \sqrt{\alpha N^{1/\alpha} \Gamma(1-\frac{1}{\alpha})}\right)$. Setting $\alpha = \log N$ makes the regret bound $O(\sqrt{T N \log N})$. Our choice of
$\alpha$ is larger than $1$ as soon as $N > 2$.

\subsubsection{Elementary calculations for bounding Frechet distribution's hazard rate}

For $\alpha > 1$, we want to show that $\sup_{x > 0} h(x) \le 2\alpha$ where
$$h(x) = \alpha x^{-\alpha-1} \frac{\exp(-x^{-\alpha})}{1-\exp(-x^{-\alpha})} .$$

First, consider the case $x \ge 1$. In this case, define $y = x^\alpha$ and note that $y \ge 1$. Then, we have
\begin{align*}
h(x) &= \frac{\alpha}{xy} \frac{\exp(-1/y)}{1-\exp(-1/y)} \le \frac{\alpha}{y} \frac{\exp(-1/y)}{1-\exp(-1/y)} \\
&\le \frac{\alpha}{y} \frac{1}{1-(1-1/(2y))} = 2\alpha .
\end{align*}
The first inequality holds because $x \ge 1$. The second holds because $\exp(-1/y) < 1$ and $\exp(-1/y) \le 1-1/(2y)$ for $y \ge 1$.

Next, consider the case $x < 1$. Define $y = 1/x$ and note that $y > 1$. Then, we have
\begin{align*}
h(x) &= \frac{\alpha}{x^{\alpha+1}} \frac{\exp(-x^{-\alpha})}{1-\exp(-x^{-\alpha})} \le \frac{\alpha}{x^{\alpha+1}} \frac{\exp(-x^{-\alpha})}{1-\exp(-1)} \\
&= \frac{\alpha}{1-e^{-1}}  y^{\alpha+1} \exp(-y^\alpha) \le 2 \alpha y^{\alpha+1} \exp(-y^\alpha) .
\end{align*}
To show an upper bound of $2\alpha$, it therefore suffices to show that $\sup_{y > 1} g(y) \le 1$ where $g(y) = y^{\alpha+1} \exp(-y^\alpha)$.
We will show this now. Note that
\[
g'(y) = (\alpha+1) y^\alpha \exp(-y^\alpha) - y^{\alpha+1} \alpha y^{\alpha-1} \exp(-y^\alpha) = y^\alpha  \exp(-y^\alpha) \left( (\alpha+1) - \alpha y^{\alpha} \right),
\]
which means that $g(y)$ is monotonically increasing on the interval $(1,y_0)$ and monotonically decreasing on the interval $(y_0,+\infty)$ where
$y_0 = \left( \frac{\alpha+1}{\alpha} \right)^{1/\alpha}$. We therefore have,
\[
\sup_{y > 1} g(y) = g(y_0) = \left( 1 + \frac{1}{\alpha} \right)^{(1+1/\alpha)} \exp\left( -(1+1/\alpha) \right) \le 2^2 \exp(-2) = 4/e^2 \le 1,
\]
where the first inequality above holds because $\alpha > 1$. Note that, for $\alpha > 1$, the function $\alpha \mapsto \left( 1 + \frac{1}{\alpha} \right)^{(1+1/\alpha)} \exp\left( -(1+1/\alpha) \right)$
decreases monotonically.

\subsection{Weibull distribution}
The CDF of Weibull is $1-\exp(-x^k)$ for $x > 0$ (and $0$ otherwise) where $k > 0$ is a shape parameter. The density is $k x^{k-1} \exp(-x^k)$ and hazard rate is $kx^{k-1}$.
For $k > 1$, hazard rate monotonically increases and is therefore unbounded for large $x$.
When $k < 1$, the hazard rate is unbounded for small values of $x$. Note that Weibull includes exponential as a special case when $k=1$.

Let $k = 1/m$ for some positive integer $m \ge 1$ and using the conditioning trick, consider a modified Weibull with CDF $1-\exp(-(x+1)^k+1)$. Density is $k (x+1)^{k-1} \exp(-(x+1)^k + 1)$ and hazard is $k (x+1)^{k-1}$ which is bounded by $k$.
When $k < 1$ we get tails heavier than the exponential but not as heavy as a Pareto or a Frechet.
The expected value of the maximum of $N$ iid (modified) Weibull random variables with parameter $k=1/m$ scales as $O(m!(\log N)^m)$ (see Appendix~\ref{sec:weibullcalc}). Thus, we get the regret bound $O(\sqrt{N T} \sqrt{m! (\log n)^m})$.
Thus, the entire modified Weibull family yields $O(\sqrt{ N \text{polylog}(N)} \sqrt{T})$ regret bounds. The best bound is obtained when $m=1$, i.e. when the Weibull becomes an exponential.